\documentclass[11pt]{article}

\usepackage{float}
\usepackage[a4paper, left=3cm,top=2.5cm,right=3cm,bottom=3cm]{geometry}

\usepackage[utf8]{inputenc}

\usepackage{authblk}

\usepackage[onehalfspacing]{setspace}

\usepackage{booktabs,multicol,multirow,tabularx}

\usepackage[english]{babel}
\usepackage{amsmath,amsthm,amssymb}
\usepackage{dsfont}
\usepackage{mathtools}
\usepackage{subfigure}

\usepackage{enumitem}
\setlist[enumerate]{nosep}   
\setlist[itemize]{nosep}

\newcommand{\com}[1]{}

\newcommand*{\R}{\mathds{R}}

\newcommand{\reg}{\mathcal{R}}

\newcommand{\rec}{\Phi}
\newcommand{\loss}{\mathcal{L}}
\newcommand{\ew}{\mathbb{E}}
\newcommand{\id}{\operatorname{Id}}

\newcommand{\Om}{\Omega}
\newcommand{\La}{\Lambda}

\newcommand{\F}{\mathcal{F}}
\newcommand{\T}{\mathcal{T}}

\newtheorem{theorem}{Theorem}
\newtheorem{condition}[theorem]{Condition}

\newtheorem{definition}[theorem]{Definiton}
\newtheorem{corollary}[theorem]{Corollary}
\newtheorem{lemma}[theorem]{Lemma}

\DeclarePairedDelimiter{\abs}{\lvert}{\rvert}
\DeclarePairedDelimiter{\norm}{\lVert}{\rVert}

\usepackage{xcolor}      
\usepackage{soul}

\newcommand\set[1]{{\left\{#1\right\}}}

\DeclareMathOperator*{\argmin}{arg\,min}

\numberwithin{equation}{section}
\numberwithin{figure}{section}
\numberwithin{theorem}{section}
\newtheorem{example}[theorem]{Example}

\allowdisplaybreaks

\usepackage{graphicx}
\usepackage{sidecap}

\title{Locally-Supervised Global Image Restoration}

\date{\today}

\author[1]{Benjamin Walder}
\author[2]{Daniel Toader}
\author[2]{Robert Nuster}
\author[2]{Günther Paltauf}
\author[3]{Peter Burgholzer}
\author[3]{Gregor Langer}
\author[4]{Lukas Krainer}
\author[1,*]{Markus Haltmeier}

\affil[1]{Universität Innsbruck, 6020 Innsbruck, Austria}
\affil[2]{Universität Graz, 8010 Graz, Austria}
\affil[3]{Research Center for Non Destructive Testing GmbH, 4040 Linz, Austria}
\affil[4]{Prospective Instruments LK OG, 6850 Dornbirn, Austria}
\affil[*]{Corresponding author: markus.haltmeier@uibk.ac.at}

\begin{document}

\maketitle

\begin{abstract} 
We address the problem of image reconstruction from incomplete measurements, encompassing both upsampling and inpainting, within a learning-based framework. Conventional supervised approaches require fully sampled ground truth data, while self-supervised methods allow incomplete ground truth but typically rely on random sampling that, in expectation, covers the entire image. In contrast, we consider fixed, deterministic sampling patterns with inherently incomplete coverage, even in expectation. To overcome this limitation, we exploit multiple invariances of the underlying image distribution, which theoretically allows us to achieve the same reconstruction performance as fully supervised approaches. We validate our method on optical-resolution image upsampling in photoacoustic microscopy (PAM), demonstrating competitive or superior results while requiring substantially less ground truth data.



\end{abstract}

\section{Introduction}\label{sec:1}

We consider the problem of reconstructing a signal or image $x \in \R^I$ from incomplete observations
\begin{equation} \label{eq:ip}
y = P_\Om x  \,,
\end{equation}
where $I$ denotes the set of all pixel indices, $\Om \subseteq I$ the subset of observed pixels and $P_\Om \colon \R^I \to \R^\Om$ the subsampling operator that restricts $x$ to measurements in $\Om$. 
When $\Om \neq I$, the restoration problem~\eqref{eq:ip} is inherently underdetermined and requires prior information on $x$ for its solution. 
This setting includes, among others, image upsampling with stride $s$, where $I = \{1, \dots, N\}^2$ and $\Om$ contains every $s$-th pixel in each dimension, as well as inpainting, where $\Om$ corresponds to arbitrary subsets of pixels for which information is missing. 
All of these cases can be described within the general framework of~\eqref{eq:ip}.

\paragraph{Prior work:}

A classical approach to solving the inverse problem~\eqref{eq:ip} is variational regularization~\cite{scherzer2009variational,aly2005image}. 
The idea is to introduce a regularization functional $\reg \colon \R^I \to [0, \infty]$ and to estimate the clean signal $x$ as minimizer of $\norm{ P_\Om x - y }_2^2 + \alpha \reg(x)$, where the parameter $\alpha > 0$ balances data fidelity and prior information. 
Common choices for $\reg$ include total variation, sparsity-promoting norms, or smoothness constraints. 
While effective in many cases, such handcrafted priors are often limited in their ability to capture the complex, data-driven structures inherent in modern imaging applications.

Recently, learning-based methods have shown superior performance. 
In the supervised setting, a reconstruction network $\rec \colon \R^\Om \to \R^I$ is trained on paired examples of $(x, y)$ to minimize the expected reconstruction error $\ew [ \norm{\rec(Y) - X}^2 ]$, 
where $X$ and $Y$ are considered as random variables subject to the forward model~\eqref{eq:ip}, and $\ew$ denotes the expectation with respect to their joint distribution. 
While highly effective, acquiring fully sampled ground truth images for all training examples is often expensive or infeasible.

Self-supervised approaches \cite{batson2019noise2self,gruber2024sparse2inverse,Hendriksen_2020,krull2019noise2voidlearningdenoising,millard2023theoretical,moran2020noisier2noise} overcome the need for fully sampled ground truth images and use only $Y$ for training. In a nutshell, these methods generate further downsampled data $P_\La Y$ and define a self-supervised reconstruction function by minimizing a self-supervised loss $\ew[\| \rec(P_\La Y) - Y \|^2]$. Theoretical results~\cite{millard2023theoretical,moran2020noisier2noise} show that such self-supervised reconstruction functions can recover the underlying image $X$. However, these results require that $\Om$ is randomly selected and, in expectation, covers the entire index set $I$. In practice, this necessitates measurements with varying sampling patterns and the ability to measure each pixel, which may be challenging. Crucially, our proposed method does not rely on randomness  of  $\Omega$ but instead adheres to a fixed, deterministic and incomplete sampling pattern.

\paragraph{Our contribution:}

To address the limitations of both supervised and self-supervised approaches, we propose a learning method that does not require a ground truth information on the full image while still enabling structured global upsampling. Instead of relying on complete ground truth images $x$, we use measurements of $x$ on only a small, fixed non-random subset $B \subseteq I$ of pixels to learn the upsampling task. The key idea is to exploit invariances inherent in the distribution of $X$, such as translational and rotational invariance. These invariances are naturally present, as an image remains the same regardless of its location or orientation. By leveraging such invariances, the observed pixels effectively generate multiple virtual training samples, enabling reconstruction quality comparable to fully supervised approaches while requiring ground truth data over a substantially smaller domain.

Specifically, our method relies on the following elements:  
\begin{itemize}
\item a fixed set $\Om \subseteq I$ used both for inference and learning,  
\item a fixed set $B \subseteq I$ used for supervision during training,  
\item a collection of translations $(T_\ell)_{\ell=1}^L$ with $(T_\ell(B))_{\ell=1}^L$ forming a partition of $I$  
\item invariance of $X$ and $\Om$ under each translation $T_\ell$.  
\end{itemize}  
{\color{black}More precise formulations of the required elements are given in Condition \ref{cond:T}. Note these assumptions, in particular, require that $B \cap \Omega^c \neq \emptyset$. Otherwise we would have $B \subseteq \Omega$, in which case no supervised data for the prediction task are collected at all. We also point out that $\Omega$ and $B$ have very different roles. On $\Omega$ we collect observed data for the actual inverse problem under study. On the other hand, on $B$ we collect additional data only during the training phase. Without essential changes, we could assume $B \cap \Omega = \emptyset$, as the supervision pixels in $B \cap \Omega$ do not provide additional information. However, we prefer the present setting where we formally aim at estimating $X$ given $X_\Omega$ instead of estimating $X_{\Omega^c}$ given $X_\Omega$, which in the noise-free inpainting case is equivalent. In particular, this excludes the extreme case $B \subseteq \Omega$, and $B \cap \Omega^c$ is required to have at least one element. However, as $(T_\ell(B))_{\ell=1}^L$ is assumed to form a partition of $I$, the set of invariances has to be large enough. Because  $(T_\ell(B \setminus \Omega))_{\ell=1}^L$ forms a partition of $I \setminus \Omega$, this gives the requirement $L \abs{B \setminus \Omega} = \abs{I \setminus \Omega}$.}

As our main theoretical result, we show that the supervised reconstruction function $\ew [ X | Y ]$ can be obtained by minimizing the loss $
\ew [ \norm{ P_B (\rec(Y)) - X_B}^2 ]$ which supervises only on $B$ rather than on the entire index set $I$.  
When $B$ constitutes only a small subset of $I$, this significantly reduces the supervision requirements.  
Moreover, as opposed to self-supervised approaches, our method does not involve any additional subsampling of the data and thus uses all available information at inference time.

{\color{black} From a practical point of view, our strategy requires more data pairs $(Y, X_B)$ and thus additional measurements $X_B$ on the supervision region $B$ during training compared to purely self-supervised approaches. However, the set of additional measurements $B \setminus \Omega$ can be much smaller than the full supervision $I \setminus \Omega$ needed in fully supervised approaches based on $(Y, X)$. Depending on the available invariances, the reduction in additionally required data can be substantial. Nevertheless, we show that our method yields the same optimal restoration function as supervised learning, which cannot be expected from purely self-supervised methods.            
}

Our strategy reduces acquisition time, eliminates the need for complex re-sampling procedures, and can be readily applied in practical experimental setups.  
We demonstrate the proposed method on image upsampling in optical-resolution photoacoustic microscopy (OR-PAM), and validate its performance on real experimental data.

\paragraph{Outline:}
The remainder of this paper is organized as follows. In Section~\ref{sec:prelim}, we introduce the studied image restoration problem and recall the main concepts of self-supervised upsampling. In Section~\ref{sec:main}, we present the proposed local supervision scheme and derive the main results. Section~\ref{sec:orpam} provides numerical experiments  on OR-PAM and comparison with global and patch-based upsampling methods. The paper concludes with a brief summary in Section~\ref{sec:conclusion}.

\section{Preliminaries}
\label{sec:prelim}

We target the image restoration problem~\eqref{eq:ip} of reconstructing a signal or image $x \in \R^I$ from incomplete observations $y = P_\Om x$, where the subsampling mask $\Om \subseteq I$ is deterministic and fixed and $I$ corresponds to all pixel locations. 

The subsampling operator is considered as  linear operator  $P_\Om \colon \R^I \to \R^\Om \colon x \mapsto (x_i)_{i \in \Om}$. Its adjoint $P_\Om^* \colon \R^\Om \to \R^I$ is referred to as the zero-upsampling operator, and its normal operator $P_\Om^* P_\Om$ as the masking operator. For $x \in \R^I$ and $y \in \R^\Om$ we have $(P_\Om^* y)_i = y_i$,  $(P_\Om^* P_\Om x)_i = x_i$ for $i \in \Om$, and $(P_\Om^* y)_i = (P_\Om^* P_\Om x)_i = 0$ for $i \notin \Om$, justifying their names. The masking operator can be written as the Hadamard (elementwise) product $P_\Om^* P_\Om x = M_\Om \odot x$ where $M_\Om \in \{0,1\}^I$ is the binary mask defined by $(M_\Om)_i = 1$ if and only if $i \in \Om$. The  restoration problem~\eqref{eq:ip} can be equivalently formulated as recovering $x \in \R^I$ from the masked version  $x_\Om = M_\Om \odot x$.  For simplicity we will proceed with the latter.

\subsection{Problem formulation}

We consider \eqref{eq:ip} in a probabilistic setting, where the unknown image $x$ is the realization of a random vector $X$ with unknown distribution. The observed (masked) data $x_\Om$ is the realization of the restricted random vector $X_\Om = M_\Om \odot x$. All expectations are taken with respect to the joint distribution $\pi$ subject to \eqref{eq:ip}, unless stated otherwise. For two random vectors $Y, Z$ we write $Y \stackrel{d}{=} Z$ and $Y = Z$ to denote equality in distribution and almost sure equality, respectively.

{\color{black}
The optimal restoration function $\hat{f} \colon \R^{I} \to \R^{I}$ is defined as the minimizer of the supervised loss function
\begin{equation}\label{eq:loss-super}
\loss(f) \coloneqq \ew\!\big[\|f(X_\Omega) - X\|^{2}\big] \,,
\end{equation}
where the minimum is taken over all equivalence classes (modulo almost-everywhere equality) of measurable functions $f \colon \R^{I} \to \R^{I}$. It is well known that the minimizer of \eqref{eq:loss-super} is given by the conditional expectation $\hat{f} = \ew[X \mid X_\Omega]$. The conditional expectation is unique up to a set of $X_\Omega$-probability zero. Thus, $\hat{f}(z)$ is well defined whenever $z = x_\Omega = M_\Omega \odot x$ for some image $x \in \R^{I}$ (again up to null sets). On the other hand, whenever $z$ does not vanish on $I \setminus \Omega$, the value $f(z)$ is not determined by \eqref{eq:loss-super} and can be chosen arbitrarily. For convenience, we still define $\hat{f}$ on the full space $\R^{I}$, keeping in mind that \eqref{eq:loss-super} specifies it only on the range of the masking operator $M_\Omega$. This choice aligns with the notation commonly used in the self-supervision literature and also reflects that the neural network architectures $(f_\theta)_{\theta \in \Theta}$ by default take inputs in $\R^{I}$. The actual function we use for inference based on data $y$ in \eqref{eq:ip} is $ \rec_\theta (y)= (f_\theta  \circ P_\Omega^{*})(y) =  f_\theta(x_\Omega)$; thus, it is not affected by how we extend $f_\theta$ outside the range of the masking operator $M_\Omega$.}

Computing the exact minimizer in \eqref{eq:loss-super} requires exact knowledge of the joint distribution $\pi$, which is usually not available. In the supervised learning paradigm, the minimizer of \eqref{eq:loss-super} is instead approximated by minimizing the empirical loss $\sum_{n=1}^N \norm{f(P_\Om^*(y_n)) - x_n}^2$ using samples $(y_n, x_n) \in \R^{\Omega} \times \R^I$ drawn independently from the joint distribution. This, however, requires fully sampled ground truth images $x_n$, which again might not be available in practice. 

Our goal is to estimate $\ew[X | X_\Om]$ without access to fully sampled instances of $X$, even during training.

\subsection{Self-supervised upsampling}

Collecting fully sampled data is often unavailable or time consuming or even impossible to acquire.  Self-supervised methods address this challenge by learning an upsampling function purely from subsampled data $y_n \in \R^\Om$ without the need for ground truth images. For self-supervised upsampling and related works see for example   \cite{batson2019noise2self,gruber2024sparse2inverse,Hendriksen_2020,krull2019noise2voidlearningdenoising,millard2023theoretical,moran2020noisier2noise}.  

In the context of image upsampling or inpainting, the self-supervision paradigm selects a second subsampling set $\Gamma \subseteq \Om$, which further degrades the partially observed data. This procedure defines synthetic training pairs $(P_\La y_n, y_n)$, where the original partially observed image $y_n$ serves as the ground truth, and $P_\La y_n$ acts as the input to the model. In the context of image restoration with known masks, the main theoretical justification is that $\ew[X \mid M_\Om \odot X]$ can, in fact, be constructed from data $(P_\La y, y)$. 

For the reader’s convenience, we recall a main theoretical result in this context due to \cite{millard2023theoretical}, where multiplicative variants of Noisier2Noise and their relation to SSDU (self-supervised learning via data undersampling, proposed in \cite{yaman2020self}) have been studied. A main recovery result states that
\begin{multline} \label{eq:ssdu-theory}
M_{I \setminus (\Om \cap \La)} \odot 
\ew [ X  \mid  M_\La \odot X_\Om ]
\\ =
M_{I \setminus (\Om \cap \La)} \odot
\Bigl( 
\argmin_{\rec} 
\ew \bigl[ \norm{ M_{I \setminus \La} \odot  
\bigl( M_\Om \odot \rec( M_\La \odot X_\Om ) - X_\Om \bigr) }^2 \mid M_\La \odot X_\Om \bigr]
\Bigr) \,,
\end{multline}
where $\Om, \La \subseteq I$ are assumed to be random subsets with $\ew[M_\Om] > 0$ and $\ew[M_\La] < 1$. \eqref{eq:ssdu-theory} states that, on the pixels outside $\Om \cap \La$, the ideal estimator $\ew[X \mid M_\La \odot X_\Om]$ for the given further masked data $M_\La \odot X_\Om$ can be found using only access to $X_\Om$ instead of $X$ during training. In our opinion, this is quite remarkable. A key assumption is the randomness of the set $\Om$, which, in expectation, covers the whole region $I$. In many practical scenarios, however, the set $\Om$ may be deterministic and fixed and, moreover, may not cover the entire imaging domain. Our method, on the other hand, allows for a fixed, non-random, incomplete design. To obtain sufficient information, we exploit invariances in the distribution of $X$.

\section{Locally-supervised global image restoration}
\label{sec:main}

The main idea for allowing supervision  on a small fixed  subset only is to exploit available invariances of the distribution. Throughout this section,  we take  $I = \{1, \dots, N\}^2$ and consider elements $x \in \R^I$ as $N$-periodic images in each dimension.  Further, $X$ is a random variable with values in $\R^I$, and $\Om \subseteq I$ is a fixed subsampling set.

\subsection{Auxiliary Results}

While we are particularly interested in translation invariance (see Definition \ref{def:trans}), we derive the main results for general linear and invertible operators $T \colon \R^I \to \R^I$.

\begin{definition}[$T$-invariant random vector]
We say that the random vector $X$ is invariant with respect to $T$ if $\mathbb{P}(X \in A) = \mathbb{P}(T(X) \in A)$ for all Borel sets $A \subseteq \R^I$.
\end{definition}

The invariance of $X$ with respect to $T$ means equality $X \stackrel{d}{=} T(X)$ in distribution, which is much weaker than equality $X = T(X)$ almost surely. For example, if $X \sim \mathcal{N}(0,1)$ and $T(X) = -x$, then $X(\omega) \neq T(X)(\omega)$ for almost every $\omega$, but their distributions coincide because the standard normal law is symmetric.

\begin{definition}[$T$-invariant subsampling]
Let $X$ be $T$-invariant.  
We say that $\Om \subseteq I$ is $T$-invariant if the pointwise-masked vector $M_\Om \odot X$ is also $T$-invariant.
\end{definition}

Thus, $\Om$ is $T$-invariant if $M_\Om \odot X \stackrel{d}{=} T(M_\Om \odot X)$.

\begin{example}[Translation invariance]
Let $X$ be invariant with respect to  horizontal translation $H$ by one pixel. If the mask $\Om$ consists of a single column, then $M_\Om \odot X$ contains only that column. Applying the translation to $M_\Om \odot X$ shifts the column by one, which is not equal in distribution to the original masked column. Hence $M_\Om$ is not $H$-invariant.   If instead the mask selects a single row, then $M_\Om \odot X$ is invariant under horizontal translation by one pixel, and thus $\Om$ is $H$-invariant.
\end{example}

The following Lemma is central to our proposal.  

\begin{lemma}\label{lem:inv}
If $X$ and $\Om$ are $T$-invariant, then $T(\ew[X | M_\Om \odot X]) = \ew[X | T(M_\Om \odot X)]$.     
\end{lemma}

\begin{proof}
Write $X_\Om = M_\Om \odot X$. Since $T$ is an invertible transform, $\sigma(T(X_\Om)) = \sigma(X_\Om)$. With the linearity of the conditional expectation, this gives $\ew[T(X) | T(X_\Om)] = \ew[T(X) | X_\Om] = T(\ew[X | X_\Om])$. Because $X$ is $T$-invariant and $\Om$ is $T$-invariant,  $(X, X_\Om) \stackrel{d}{=} (T(X), T(X_\Om))$. Hence  $\ew[T(X) | T(X_\Om)] = \ew[X | T(X_\Om)]$, which concludes the proof. 
\end{proof}

\begin{definition}
A function $ f \colon \R^I \to \R^I$ is called $T$-equivariant if $f \circ T  =   T  \circ f$.    
\end{definition}

Thus Lemma~\ref{lem:inv} shows that the optimal upsampling function $\ew[X | M_\Om \odot X]$ is $T$-equivariant. In other words, applying $T$ before or after upsampling via $f$ yields the same result.

\subsection{Main results}
\label{sec:main-res}

Our aim is the determination of $\ew[X | X_\Om]$ from access to the ground truth $X$ on a subset $B \subseteq I$ only. This will be achieved under the following assumptions.  

\begin{condition}[Locally-supervised global restoration framework] \mbox{} \label{cond:T}
\begin{enumerate}[label=(A\arabic*)]
\item \label{cond:T1} $\T \coloneqq (T_\ell)_{\ell=1}^L$ is a family of linear and invertible operators on $\R^I$  with $T_1= \id$. 
\item \label{cond:T2} $\Om, B \subseteq I$ are fixed subsampling and supervision sets, respectively. 
\item \label{cond:T3} $X$ and $\Om$ are $T_\ell$-invariant for all $\ell \in \set{1, \dots, L}$.  

\item \label{cond:T4} $\sum_{\ell=1}^LT_\ell^{-1}\big(M_B\odot T_\ell(X)\big)=X$ with $T^{-1}_\ell$ being the inverse operator to $T_\ell$.

\item \label{cond:T5} $X_\Om = M_\Om \odot X$ is the observed data.
\end{enumerate}
\end{condition}

{\color{black}
Let us briefly discuss Condition \eqref{cond:T}.  Items \ref{cond:T1}, \ref{cond:T2}, and \ref{cond:T5} determine the setting we work in and the observed data, and are therefore not restrictive. By contrast, \ref{cond:T3} and \ref{cond:T4} may appear to impose potentially restrictive conditions on the supervision set $B$ and on the class $\mathcal{T}$. However, this is not the case; we consider the interplay between these conditions, as reflected in \ref{cond:T3} and \ref{cond:T4}, to be quite natural. Our aim is to reduce the supervision set from the full set $I$ of pixels to a potentially small one, and the way to do so is to exploit invariances present in the dataset $X$. If there is no invariance, our method cannot reduce the supervision set. Formally, with only the identity mapping $T_1 = \id$ as an invariance, one must take $\Omega = I$; nothing is gained. If, however, there is at least one additional (nontrivial) invariance $T_2$, we may consider $\T = \{T_1, T_2\}$. As a concrete example, let $T_2$ be a periodic shift by $N/2$ pixels in the horizontal dimension. Then we can take the supervision set $\Omega$ as the left half of the image, and \ref{cond:T3} and \ref{cond:T4} are satisfied. The more general case simply allows multiple invariances. While having more invariances may be harder to achieve in practice, in that case \ref{cond:T4} becomes easier to satisfy.
}

A function $f \colon \R^I \to \R^I$ is called $\T$-equivariant if it is $T_\ell$-equivariant for all $\ell = 1, \dots, L$. The set of all measurable $\T$-equivariant functions will be denoted by $\F(\R^I;\T)$.

\begin{theorem}[Locally-supervised global image restoration] \label{thm:main}
Let $\T$, $\Om$, $B$, $X$, $X_\Om$ satisfy \ref{cond:T1}-\ref{cond:T5}. Then $\ew[X | X_\Om]$ is $\T$-equivariant and
\begin{equation} \label{eq:main}
\ew[X | X_\Om] = \argmin_{f \in \F(\R^I;\T)} \ew\bigl[\norm{M_B \odot f(X_\Om) - X_B}^2\bigr] \,.
\end{equation}
\end{theorem}

\begin{proof}
The equivariance of $\ew[X | X_\Om]$ follows from Lemma~\ref{lem:inv}.
Further, from the $\T$-equivariance of $\ew[X | X_\Om]$ and conditions \ref{cond:T3}, \ref{cond:T4}, we get
\begin{align*}
\ew[X | X_\Om] &= 
\argmin_{f \in \F(\R^I;\T)} 
\ew \bigl[ \norm{f(X_\Om) - X}^2 \bigr] \\
&= 
\argmin_{f \in \F(\R^I;\T)} 
\biggl[ \sum_{\ell}
\ew \norm{ M_{B} \odot T_\ell(f(X_\Om)) - M_{B} \odot T_\ell(X) }^2 \biggr] \\
&= 
\argmin_{f \in \F(\R^I;\T)} 
\biggl[ \sum_{\ell}
\ew \norm{ M_{B} \odot f(T_\ell(X_\Om)) - M_{B} \odot T_\ell(X) }^2 \biggr] \\
&= 
\argmin_{f \in \F(\R^I;\T)} 
\biggl[ \sum_{\ell}
\ew \norm{M_B \odot f(X_\Om) - M_B \odot X }^2 \biggr] \\
&= 
\argmin_{f \in \F(\R^I;\T)} 
\ew\bigl[ \norm{ M_B \odot f(X_\Om) - M_B \odot X }^2\bigr],
\end{align*}
which is~\eqref{eq:main}.
\end{proof}

{\color{black} Recall that $\hat{f} = \mathbb{E}\!\left[X \mid X_{\Omega}\right]$ is the minimizer of the supervised loss \eqref{eq:loss-super}, which is unique up to sets of $X_{\Omega}$-probability zero. Essential non-uniqueness can arise only when $\hat{f}$ is evaluated at elements $z$ that have zero probability of occurring as $X_{\Omega}$. We emphasize that in such a situation the right-hand side of \eqref{eq:main} also does not provide additional information about $z$. This can be interpreted as follows: for strictly out-of-distribution samples, neither of the losses \eqref{eq:loss-super} or \eqref{eq:main} provides any information and any prediction is feasible. Conversely, in all other cases where predictions are made on feasible elements, both loss functions yield the same unique minimizer.}

\subsection{Translation-invariant case}

\begin{definition}[Translation Operators] \label{def:trans}
We define the translation operators in the horizontal and vertical directions, $H, V \colon \R^I \to \R^I$, by
\begin{align}
(H(x))_{i_1, i_2} &= x_{i_1+1, i_2},\\
(V(x))_{i_1, i_2} &= x_{i_1, i_2+1},
\end{align}
where all indices are taken $N$-periodically.
We denote by $\T_{H,V}$ the set of all translations  $T = H^a \circ V^b$ for some $a, b =\set{0, \dots,   N-1}$.  
\end{definition}

\begin{condition}[Translation-invariant framework] \mbox{} \label{cond:T}
\begin{enumerate}[label=(B\arabic*)]
\item \label{cond:TI1} $\T \coloneqq (T_\ell)_{\ell=1}^L$ is a family of translations  in $\T_{H,V}$   with $T_1= \id$. 
\item \label{cond:TI2} $\Om, B \subseteq I$ are fixed subsampling and supervision sets, respectively. 
\item \label{cond:TI3} $X$ and $\Om$ are $T_\ell$-invariant for all $\ell \in \set{1, \dots, L}$.  

\item \label{cond:TI4} $\{T_\ell(B)\}_{\ell=1}^L$ is a partition of $I$.

\item \label{cond:TI5} $X_\Om = M_\Om \odot X$ is the observed data.
\end{enumerate}
\end{condition}

\begin{corollary}[Locally-supervised global image restoration] \label{cor:main}
Let $\T$, $\Om$, $B$, $X$, $X_\Om$ satisfy \ref{cond:TI1}-\ref{cond:TI5}. Then $\ew[X | X_\Om]$ is $\T$-equivariant and
\begin{equation} \label{eq:main-TI}
\ew[X | X_\Om] = \argmin_{f \in \F(\R^I;\T)} \ew\bigl[\norm{M_B \odot f(X_\Om) - X_B}^2\bigr] \,.
\end{equation}
\end{corollary}

\begin{proof}
According to \ref{cond:TI1}, any $T_\ell$ is a linear and invertible operator, and \ref{cond:TI4} implies \ref{cond:T4}. Thus, Corollary~\ref{cor:main} is an immediate consequence of Theorem~\ref{thm:main}.
\end{proof}

Theorem~\ref{thm:main} and Corollary~\ref{cor:main} show that supervision on the subset $B$ is sufficient to obtain the ideal restoration function $\ew[X | X_\Om]$ on the whole domain. According to conditions~\ref{cond:T4}, \ref{cond:TI4} , the more invariances we have and exploit, the smaller the supervision set can be.  In the application presented in Section~\ref{sec:orpam}, we use 4 translations, which reduces the number of pixels required for supervision by a factor of 4.

{\color{black} Let us discuss some practical implications and features of the equality \eqref{eq:main-TI}. When we choose the supervision set $B \subsetneq I$, the optimal restoration function $\mathbb{E}\!\left[X \mid X_{\Omega}\right]$ can be learned using only instances of $X_{B}$, which form a smaller dataset than instances of the original $X$. In particular, this avoids collecting any data on the complementary domain $I \setminus B$, which is crucial if the complementary domain is inaccessible due to practical constraints, or if collecting data on different domains requires adjusting the device, which may be costly, time-consuming, and error-prone. Besides these practical advantages, it introduces the numerical challenge of minimizing over $\mathcal{T}$-equivariant functions $f$  rather than arbitrary functions. Several strategies are possible to guarantee $\mathcal{T}$-equivariance during numerical optimization. In the numerical implementation presented below, we start with a standard architecture (specifically, a U-Net) and then adjust the main building blocks (in this case, the convolutions) to be translation-equivariant, for example by enforcing appropriate boundary conditions. Other possible choices are currently under investigation and will be presented in future work.}

\section{Application to accelerate OR-PAM}
\label{sec:orpam}

In this section, we present an application of our theory to OR-PAM. We first provide a brief recap of OR-PAM and then present  specific sampling strategies together with numerical results and comparison.

\subsection{OR-PAM working principle}

OR-PAM is a high-resolution imaging modality that leverages the photoacoustic effect to visualize optical absorption contrasts near the surface of biological tissues with micrometer-scale lateral resolution \cite{PARK2025100739, ParkJeongwoo25}. The underlying mechanism of OR-PAM is the photoacoustic effect: when short laser pulses are absorbed by tissue chromophores such as hemoglobin, melanin, DNA/RNA or lipids, rapid thermoelastic expansion induces ultrasonic waves. These pressure waves are then detected by a focused ultrasonic transducer placed in proximity to the sample. In contrast to standard acoustic-resolution PAM, where spatial resolution is limited by the acoustic focus, OR-PAM achieves superior lateral resolution by optically focusing the excitation beam to a diffraction-limited spot.  
The typical transmission mode setup and explanation are shown in Figure \ref{fig:1}; a precise description is given in \cite{paltauf2023model}.

\begin{SCfigure}[2][htb!]
    \centering
    \includegraphics[width=0.31\linewidth]{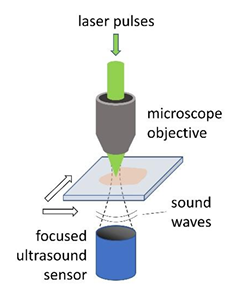}
    \caption{In OR-PAM, laser pulses from a pulsed light source and an ultrasound detector are focused on the same axis and collect highly resolved, pixel-wise image information of the sample surface down to a depth limited by the ballistic penetration depth. Taking the maximum amplitude value at a specific point in time or defined time window from the measured temporal signals at each scan position, one obtains a 2D image x of the absorbing structures close to the surface. Raster scanning is time-consuming, and the number of sampling points represents a trade-off between speed and sufficient sampling density.}
    \label{fig:1}
\end{SCfigure}

In a nutshell, the OR-PAM measurement setup provides an image $x \in \R^I$, where the data at each pixel $i \in I$ requires a specific measurement by rastering the measurement beam along the probe surface. However, scanning each pixel separately is time-consuming and critical in time-sensitive applications. Moreover, collecting many samples may damage the object. Thus, accelerating the process by scanning only a subset $\Om \subseteq I$ of all pixels is beneficial in many respects. This leads to the upsampling problem \eqref{eq:ip}.

\subsection{Sparse-Dense OR-PAM}

We employ the sparse-dense sampling scheme proposed in \cite{walder2025self}, a regular sampling scheme that is invariant with existing OR-PAM settings. We assume that full sampling is used with an even number of pixels in each dimension, and take the sub-sampling set to consist of every second pixel in each dimension. Supervised measurements are only taken in one fixed quadrant, here chosen to be the upper left corner. Thus we have               
\begin{align*}
I &\coloneqq \set{1, \dots, N} \times \set{1, \dots, N},\\ 
\Om &\coloneqq \set{1, 3, \dots, N-1} \times \set{1, 3, \dots, N-1}, \\
B &\coloneqq \set{1, \dots, N/2} \times \set{1, \dots, N/2} \,. 
\end{align*}
Data $X_\Omega$ (sparsely subsampled) is then used for inference, and pairs $(X_B, X_\Omega)$ for training, referred to as sparse-dense sampling.

In order to apply the theory of the previous section, we need translation operators $(T_\ell)_\ell$ such that $\{T_\ell(B)\}_\ell$ forms a partition of $I$. For that, we can choose four translations:
\begin{itemize}
  \item $T_1 = \id$ (the identity),
  \item $T_2 = H^{N/2}$ (horizontal translation by $N/2$ pixels),
  \item $T_3 = V^{N/2}$ (vertical translation by $N/2$ pixels),
  \item $T_4 = H^{N/2} \circ V^{N/2}$ (translation by $N/2$ pixels in both directions).
\end{itemize}
{\color{black} The partition condition \ref{cond:TI4} is in this case readily verified. Choices without additional supervised data, such as $\tilde B = \Omega$ or $\tilde B  = \Omega \cap B$, are not possible, as this would contradict \ref{cond:TI4}.}
Other invariances are also possible, for example, rotating the image by multiples of $90$ degrees; however, in this paper, we stick to translation invariances only. {\color{black}  The numerical realization of these broader invariances presents additional challenges; appropriate strategies will be developed in future work.}

{\color{black} Recall the fundamental difference between the observation set $\Omega$ and the supervision set $B$. The observation set $\Omega$ is required to be invariant under the transformations $(T_\ell)_\ell$. The supervision set, on the other hand, is required to break this invariance: applying $(T_\ell)_\ell$ to $B$ should cover the full image domain $I$. The aim of our sparse–dense sampling scheme is to reduce the size of the supervision region compared to the full supervision region $I$ used in standard supervised learning. To this end, we exploit the invariances $T_1, T_2, T_3, T_4$.}

Let $X \in \R^I$ model fully sampled OR-PAM images with $\T = (T_1, T_2, T_3, T_4)$-invariant distribution. Then, according to Corollary~\ref{cor:main}, we have $\ew[X | X_\Om] = \argmin_{f \in \F(\R^I;\T)} \norm{M_B \odot f(X_\Om) - X_B}^2$.  
Thus, the ideal restoration function $\ew[X | X_\Om]$ can be obtained from locally supervised data $X_B$ from a fixed design. To realize $\ew[X | X_\Om]$, we generate random pairs $(x_{B, n}, x_{\Om, n})$ and minimize the empirical risk  
\begin{equation}\label{eq:inference}
   \loss(\theta) \coloneqq \sum_n \norm{M_B \odot f_\theta(x_{\Om, n}) - x_{B, n}}^2
\end{equation}
over a parameterized class of $\T$-equivariant functions $f_\theta \colon \R^I \to \R^I$.

\subsection{Implementation details}

\paragraph{Architecture:}
In our numerical experiments, we employ a U-Net-based architecture \cite{jiangtao2025comprehensivereviewunetvariants} applied to  zero-filled data. The architecture consists of 21 convolutional layers, organized into three downsampling blocks (each comprising two convolutional layers followed by a strided convolution for downsampling), a bottleneck with two convolutional layers and a dropout layer, and three upsampling blocks each of them composed of one transposed convolutions for upsampling and two convolutional layers. The dropout layer at the stage of the bottleneck zeros out feature maps with a $50\%$-chance, which avoids overfitting to the training data. After every convolution a ReLu activation function is implemented to introduce nonlinearities. Skip connections link corresponding encoder and decoder stages. All that is implemented in python with the pytorch package and Adam optimization \cite{kingma2017adammethodstochasticoptimization}.

\paragraph{Equivariance:}
{\color{black} 
To ensure that the network satisfies the condition of a $\T$-equivariant function, we analyze the translation-invariance properties of the architecture in detail and adapt elements when necessary. Specifically, this means the following.} 

\begin{itemize}
    \item  {\color{black}  Without boundary effects, convolutional layers are translation-equivariant because the same filter is applied with shared weights across all spatial positions, so a shift in the input produces an equivalent shift in the feature maps. When restricting to a finite domain with standard settings, this property holds only in the interior of the image; at the boundaries, padding and edge effects may break perfect equivariance. To overcome this issue, in our implementation all convolutional layers use \texttt{padding\_mode=circular}, which ensures that the convolution operator is translation-equivariant with periodic boundary conditions.} For convolutions with stride \(s = 1\), translation invariance holds for all integer shifts. However, the encoder path contains three strided convolutional layers (\texttt{stride}=2) acting as downsampling operators. Similarly, the decoder path contains three strided transposed convolutional layers (\texttt{stride}=2) acting as upsampling operators. For these layers, translation invariance is guaranteed only for translations that are multiples of the stride. Otherwise, aliasing effects occur, which break exact invariance. Since there are three downsampling operations in the encoder, the overall stride is $2\cdot2\cdot2=8$.

\item All nonlinearities in the network are implemented using the ReLU activation function. Since ReLU acts pointwise on each spatial position and channel, it commutes exactly with translations. Concatenation operations in the skip connections operate along the channel axis only. Therefore, they preserve translation invariance. Similarly, dropout is applied via \texttt{Dropout2d}, which zeroes entire channels using a spatially constant mask, and thus also commutes with translations for any fixed mask.
\end{itemize}

In conclusion, the architecture used is indeed equivariant with respect to translations 
\( T = H^a \circ V^b \) for shifts \( a, b \) that are multiples of 8 pixels.

\paragraph{Training:}
The U-Net was trained using the mean squared error (MSE) as the loss function. 
Optimization was performed using the Adam optimizer ($\eta = 0.0004$), combined with a \texttt{ReduceLROnPlateau} scheduler 
(factor $0.5$, patience $8$, threshold $10^{-6}$) to improve convergence speed. 
During training, both predictions and targets were cropped to the supervision region, where the loss was evaluated. 
The best model was selected based on the lowest validation loss, with early stopping applied once the validation loss dropped below $10^{-10}$. 
After training, the selected model was evaluated on the test set. 
For all experiments, the neural network was trained for 80 epochs.

\begin{figure}[h!]
    \begin{subfigure}
    \centering
    \includegraphics[width=0.235\linewidth]{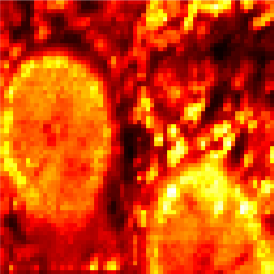}
    \end{subfigure}
    \begin{subfigure}
    \centering
    \includegraphics[width=0.235\linewidth]{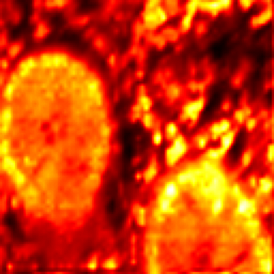}
    \end{subfigure}
    \begin{subfigure}
    \centering
    \includegraphics[width=0.235\linewidth]{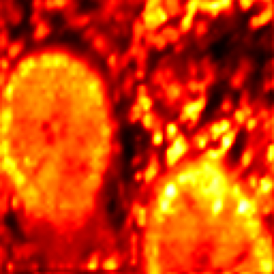}
    \end{subfigure}
    \begin{subfigure}
    \centering
    \includegraphics[width=0.235\linewidth]{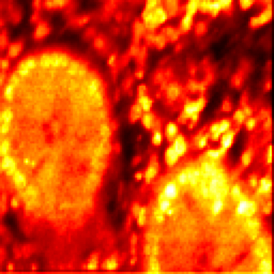}
    \end{subfigure}

    \begin{subfigure}
    \centering
    \includegraphics[width=0.235\linewidth]{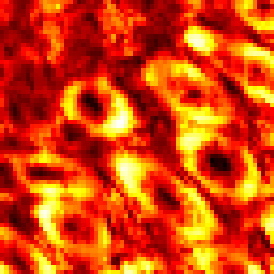}
    \end{subfigure}
    \begin{subfigure}
    \centering
    \includegraphics[width=0.235\linewidth]{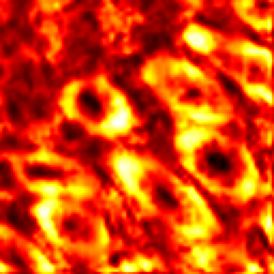}
    \end{subfigure}
    \begin{subfigure}
    \centering
    \includegraphics[width=0.235\linewidth]{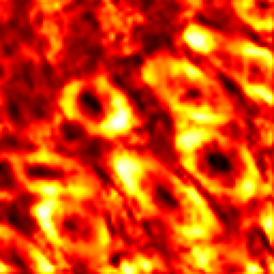}
    \end{subfigure}
    \begin{subfigure}
    \centering
    \includegraphics[width=0.235\linewidth]{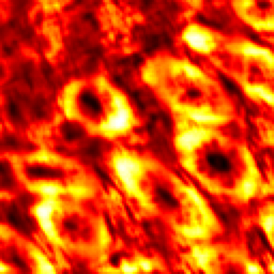}
    \end{subfigure}

    \begin{subfigure}
    \centering
    \includegraphics[width=0.235\linewidth]{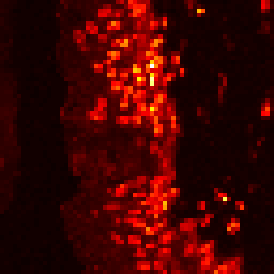}
    \end{subfigure}
    \begin{subfigure}
    \centering
    \includegraphics[width=0.235\linewidth]{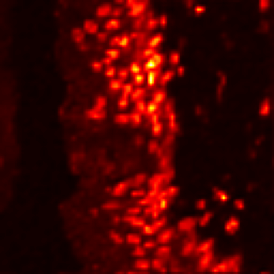}
    \end{subfigure}
    \begin{subfigure}
    \centering
    \includegraphics[width=0.235\linewidth]{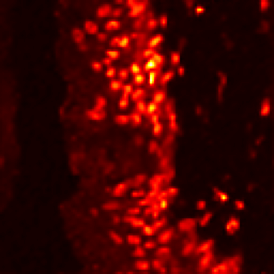}
    \end{subfigure}
    \begin{subfigure}
    \centering
    \includegraphics[width=0.235\linewidth]{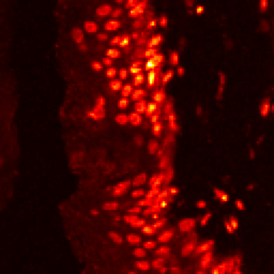}
    \end{subfigure}
\caption{Visualization of the performance of a neural network trained with sparse-dense training images. First column: Image of measured pixels only with size $64\times64$ pixels. Second column: Output image of a network trained with sparse-dense training images. Third column: Output image of a network trained fully supervised. Fourth column: Ground truth image of size $128\times128$ pixels.}
\label{Ex:1}
\end{figure}

\subsection{Numerical Studies}

The investigated samples consist of human lung tissue sections. As they originate from different anatomical regions, they exhibit considerable variability in both cellular architecture and structural organization. For the visualization and test set, we selected images that capture this variability as comprehensively as possible. The  data set is split into 154 images for the training set and 64 images for the validation set. The full-resolution images show the absorption contrast of the sample surfaces and are of size $128\times128$ pixels.

\paragraph{Evaluation against supervised learning:}
Following the sparse-dense sampling from above, we only need to generate sparsely sampled images $X_\Omega$ and supervision images $X_B$ of each of the tissues. This saves $9/16$ of measurements compared to the fully supervised training method and only relies on one quadrant of full measurements. Pixels that are not measured in this procedure are substituted with zeros. As demonstrated in Figure~\ref{Ex:1}, the proposed method performs nearly as well as the fully supervised one.

\paragraph{Decrease of supervision area:}
In theory, the size of the supervision set $B$ can be chosen arbitrarily small as long as the underlying distribution of the images is equivariant to all translation operators $T_l$ such that $\sum_{\ell=1}^LT_\ell^{-1}\big(M_B\odot T_\ell(X)\big)=X$. Since the restoration function is also required to be translation-equivariant, and the architecture of our neural network guarantees this property only for multiples of $8$ pixels, we cannot restrict the supervision patch further than that, while strictly sticking to theory. Therefore, we choose
\begin{align*}
I &\coloneqq \set{1, \dots, N} \times \set{1, \dots, N},\\ 
\Om &\coloneqq \set{1, 3, \dots, N-1} \times \set{1, 3, \dots, N-1}, \\
B &\coloneqq \set{1, \dots, N/16} \times \set{1, \dots, N/16} \,,
\end{align*}
and assume the distribution to be equivariant to all translations
\(T_{\ell, k}= H^{\ell \cdot N/16} \circ V^{k \cdot N/16}\)
for $\ell,k\in\{0,\ldots,15\}$. This corresponds to a supervision patch of $8\times8$ pixels. As Figure~\ref{Ex:2} shows, it performs similarly well as a network trained with a $64$ times larger supervision patch.  Deviating from the strict theoretical requirements, we obtained surprisingly good results even when choosing $B$ to be the bare minimum,
\( B \coloneqq \set{1, \dots, N/64} \times \set{1, \dots, N/64} \), 
which corresponds to a patch of $2 \times 2$ pixels. This minimal choice still yields satisfactory results, as also illustrated in Figure~\ref{Ex:2}. The error plot in Figure~\ref{fig:error} illustrates the performance of neural networks trained with varying supervision set sizes. The results indicate that the network performance remains largely unchanged up to a supervision set size of $4 \times 4$ pixels.

\paragraph{Fixed number of supervision pixels:}
Furthermore, it is interesting to examine how the error changes when varying the size of the supervision set while keeping the total number of supervised pixels approximately constant by adjusting the number of training images. As illustrated in Figure~\ref{fig:error_2}, the error increases as the number of training images decreases, even though the overall number of supervised pixels remains more or less unchanged. This observation provides an initial indication of the relevance of pixels outside the supervision patch, as will be discussed below.
\begin{figure}

    \begin{subfigure}
    \centering
    \includegraphics[width=0.235\linewidth]{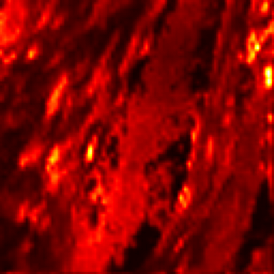}
    \end{subfigure}
    \begin{subfigure}
    \centering
    \includegraphics[width=0.235\linewidth]{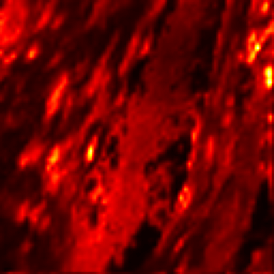}
    \end{subfigure}
    \begin{subfigure}
    \centering
    \includegraphics[width=0.235\linewidth]{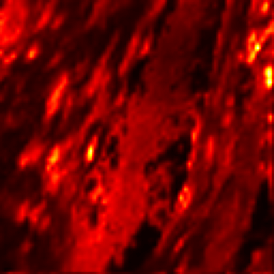}
    \end{subfigure}
    \begin{subfigure}
    \centering
    \includegraphics[width=0.235\linewidth]{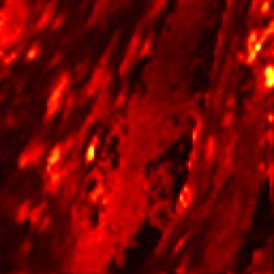}
    \end{subfigure}
 
    \begin{subfigure}
    \centering
    \includegraphics[width=0.235\linewidth]{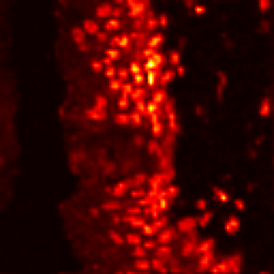}
    \end{subfigure}
    \begin{subfigure}
    \centering
    \includegraphics[width=0.235\linewidth]{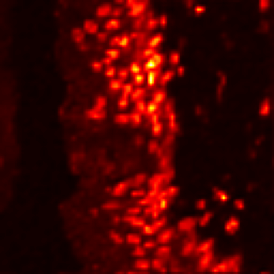}
    \end{subfigure}
    \begin{subfigure}
    \centering
    \includegraphics[width=0.235\linewidth]{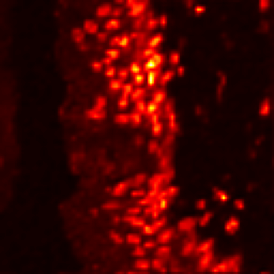}
    \end{subfigure}
    \begin{subfigure}
    \centering
    \includegraphics[width=0.235\linewidth]{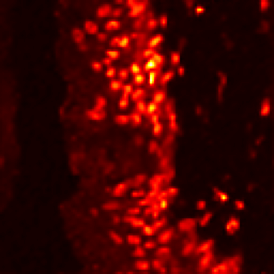}
    \end{subfigure}
\caption{Visualization of the performance of neural networks trained with different supervision patch sizes. First column: Output image of a network trained with a supervision patch of size $2\times2$ pixels. Second column: Output image of a network trained with a supervision patch of size $8\times8$ pixels. Third column: Output image of a network trained with a supervision patch of size $64\times64$ pixels. Fourth column: Output image of a network trained fully supervised.}
\label{Ex:2}
\end{figure}
\begin{figure}
    \centering
    \includegraphics[width=0.8\linewidth]{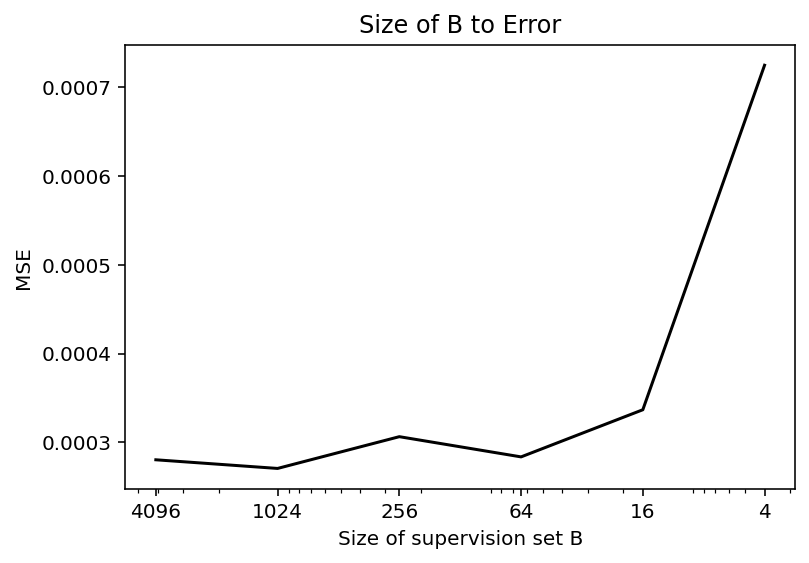}
    \caption{Relationship between the size of the supervision set $B$ and the mean squared error (MSE). A reduction in $B$ leads only to a small increase in error until a supervision set size of $4\times4$ pixels. The reported values represent the mean across $5$ different test images.}
    \label{fig:error}
\end{figure}
\begin{figure}
    \centering
    \includegraphics[width=0.8\linewidth]{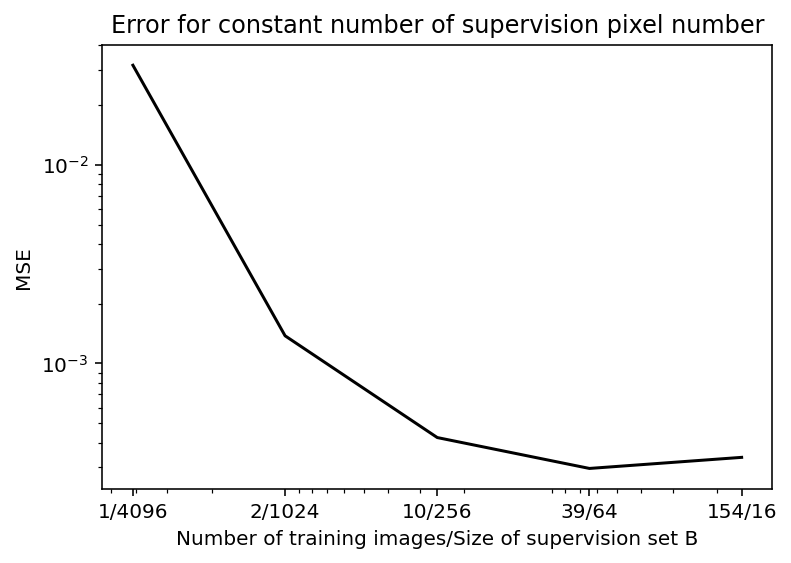}
    \caption{Mean squared error (MSE) as a function of the number of training images for a constant total number of supervised pixels. The reported values represent the mean across $5$ different test images.}
    \label{fig:error_2}
\end{figure}

\paragraph{Evaluation against patch-wise upsampling:}
 Naturally, the question arises whether it is necessary to measure pixels outside of the supervision patch or if it is possible to train a network purely with the information of the supervision patch. Our network architecture is translation-equivariant by construction and independent of the input size. 
Consequently, training can be performed by providing a downsampled version of a patch as input, while the loss is computed by comparing the output with the corresponding ground-truth supervision patch. During inference, the same filters are applied across the entire image in a single pass. This procedure is equivalent to a sliding-window evaluation, but computationally more efficient. The main limitation is that the effective receptive field restricts the available context.
\begin{table}
    \centering
    \begin{tabular}{l|c|c|c}
        & MSE & SSIM & PSNR \\
        Fully supervised & $2,0\cdot10^{-4}\pm1,5\cdot10^{-4}$ & $0,968\pm0.011$ & $38,3\pm3,8$\\
        \hline
        \hline
        Bilinear interpolation & $6,1\cdot10^{-4}\pm6,5\cdot10^{-4}$ & $0,919\pm0,033$ & $34,4\pm5,4$ \\
        \hline
        Patch supervised & $3,6\cdot10^{-4}\pm3,2\cdot10^{-4}$ & $0,947\pm0,014$ & $35,9\pm4,1$ \\
        \hline
        Sparse-dense & $\mathbf{2,8\cdot10^{-4}\pm 2,5\cdot10^{-4}}$ & $\mathbf{0,961\pm0,008}$ & $\mathbf{36,7\pm 3,6}$
    \end{tabular}
    \caption{Comparison of the mean squared error (MSE), structural similarity index measure (SSIM), and peak signal-to-noise ratio (PSNR) \cite{Dohmen2025} between bilinear interpolation and two neural networks: one trained using the proposed method and the other trained solely with information from the supervision patches. The reported values represent the mean across $5$ different test images}
    \label{Tab:1}
\end{table}

\begin{figure}
    \begin{subfigure}
    \centering
    \includegraphics[width=0.235\linewidth]{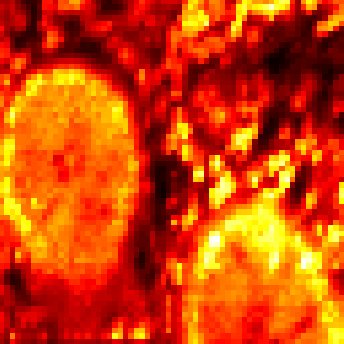}
    \end{subfigure}
    \begin{subfigure}
    \centering
    \includegraphics[width=0.235\linewidth]{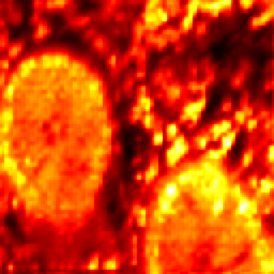}
    \end{subfigure}
    \begin{subfigure}
    \centering
    \includegraphics[width=0.235\linewidth]{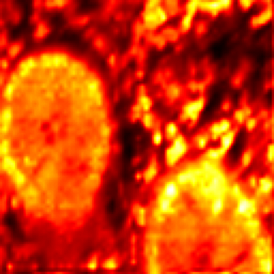}
    \end{subfigure}
    \begin{subfigure}
    \centering
    \includegraphics[width=0.235\linewidth]{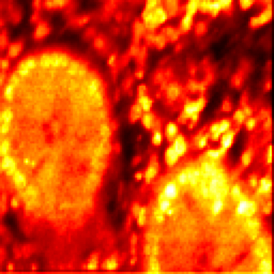}
    \end{subfigure}

    \begin{subfigure}
    \centering
    \includegraphics[width=0.235\linewidth]{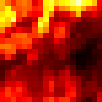}
    \end{subfigure}
    \begin{subfigure}
    \centering
    \includegraphics[width=0.235\linewidth]{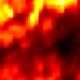}
    \end{subfigure}
    \begin{subfigure}
    \centering
    \includegraphics[width=0.235\linewidth]{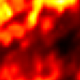}
    \end{subfigure}
    \begin{subfigure}
    \centering
    \includegraphics[width=0.235\linewidth]{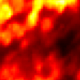}
    \end{subfigure}

    \begin{subfigure}
    \centering
    \includegraphics[width=0.235\linewidth]{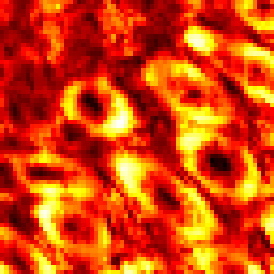}
    \end{subfigure}
    \begin{subfigure}
    \centering
    \includegraphics[width=0.235\linewidth]{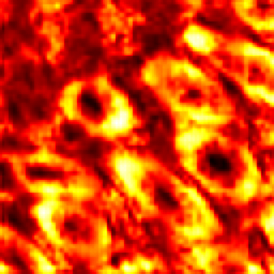}
    \end{subfigure}
    \begin{subfigure}
    \centering
    \includegraphics[width=0.235\linewidth]{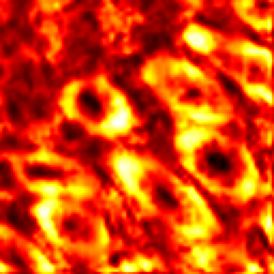}
    \end{subfigure}
    \begin{subfigure}
    \centering
    \includegraphics[width=0.235\linewidth]{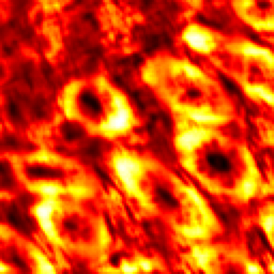}
    \end{subfigure}

    \begin{subfigure}
    \centering
    \includegraphics[width=0.235\linewidth]{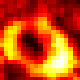}
    \end{subfigure}
    \begin{subfigure}
    \centering
    \includegraphics[width=0.235\linewidth]{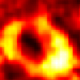}
    \end{subfigure}
    \begin{subfigure}
    \centering
    \includegraphics[width=0.235\linewidth]{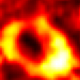}
    \end{subfigure}
    \begin{subfigure}
    \centering
    \includegraphics[width=0.235\linewidth]{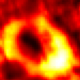}
    \end{subfigure}
\caption{Evaluation against patch-wise image restoration. Rows 2 and 2 are zoomed in versions of the images in the rows 1 and 3. First column: Image of measured pixels only with size $64\times64$ pixels. Second column: Output image of a network trained supervised on patches of size $8\times 8$. Third column: Output image of a network trained with a supervision patch of size $8\times8$ pixels. Fourth column: Ground truth image of size $128\times128$.}
\label{Ex:3}
\end{figure}

Starting with a supervision patch of size $2 \times 2$, it is not even possible to train the network we chose, as its architecture includes three downsampling operations. One could, in principle, employ a different architecture with only a single downsampling step. However, this would lead to substantially higher computational costs, since the resulting feature maps would be significantly larger. For this reason, we begin by comparing networks trained with supervision patches of size $8 \times 8$ pixels. In this setting, we already observe superior performance for networks trained on sparse-dense images. The network’s field of view is much bigger in this case, which has a pronounced impact on the overall performance of the neural networks. Although the differences may appear minor on the compared images in lines one and three of Figure~\ref{Ex:3}, one can clearly see that the network trained solely on patches struggles to get rid of the pixel structure of the input image and is far from the performance our proposed method is capable of in the zoomed in versions. This is also statistically shown in Table~\ref{Tab:1}, where both methods are compared with bilinear interpolation using three different metrics. Since the samples, and consequently the corresponding images, exhibit substantial variability, we selected our test dataset to ensure that all image types are represented. As a result, the standard deviation is relatively high: images with fewer structural features are comparatively easy to reconstruct, whereas images containing dense cellular structures pose a greater challenge.

\paragraph{Superior performance of over patch-wise upsampling:}

The reason why the proposed method performs better than a network trained only with information of the supervised patch, is that it does not have to rely solely on local information. To illustrate this, we generate artificial images with size $16\times16$ pixels consisting of four quadrants. Each quadrant contains one of the following patterns:  
\begin{enumerate}[label=(P\arabic*)]
    \item horizontal lines,  \label{P1}
    \item vertical lines,  \label{P2}
    \item a checkerboard pattern,  \label{P3}
    \item completely black with a white square in the upper left corner. \label{P4}
\end{enumerate}

The pattern of each quadrant is selected with equal probability. However, the choice is not independent: if a quadrant displays pattern \ref{P4}, then the quadrant horizontally next to it shows pattern \ref{P1}. The diagonal one shows pattern \ref{P2} and the quadrant vertically next to is shows pattern \ref{P3}. Consequently, once the pattern of a single quadrant is determined, the patterns of all remaining quadrants are uniquely fixed. An example is shown in Figure~\ref{Ex:4} (image 1 and 2). Since all quadrants are equally distributed, the underlying distribution of the images is invariant to $T_1 = \id$, $T_2 = H^{N/2}$, $T_3 = V^{N/2}$, $T_4 = H^{N/2} \circ V^{N/2}$.
We use  $\Om  = \set{1, 3, \dots, N-1}^2$ for the downsampling set and  
$B = \set{1, \dots, N/2}^2$ for the supervision set.
In particular, the downsampled versions of patterns \ref{P1}, \ref{P2}, \ref{P3} coincide. As shown in images 3 and 4 in Figure~\ref{Ex:4} the sparse-dense training data perfectly recovers the structure of the whole image, while  the network trained solely on patches completely fails to recover the missing parts.

\begin{figure}
    \begin{subfigure}
    \centering
    \colorbox{black!50}{\includegraphics[width=0.22\linewidth]{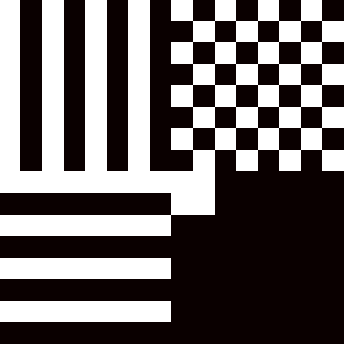}}
    \end{subfigure}
    \begin{subfigure}
    \centering
    \colorbox{black!50}{\includegraphics[width=0.22\linewidth]{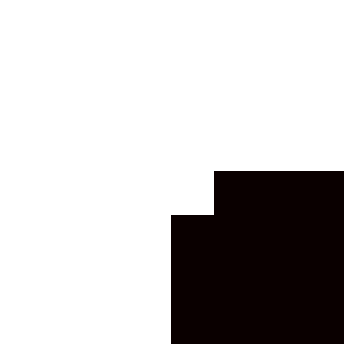}}
    \end{subfigure}
    \begin{subfigure}
    \centering
    \colorbox{black!50}{\includegraphics[width=0.22\linewidth]{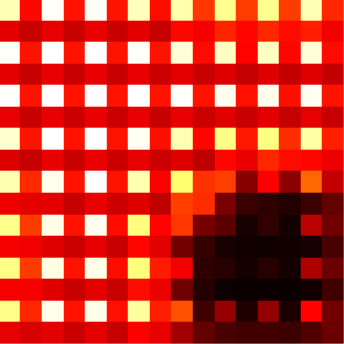}}
    \end{subfigure}
    \begin{subfigure}
    \centering
    \colorbox{black!50}{\includegraphics[width=0.22\linewidth]{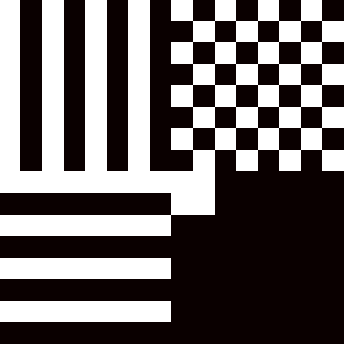}}
    \end{subfigure}
\caption{Visualization of superior performance of sparse-dense sampling over patch-wise supervised learning. From left to right: ground truth (image 1); downsampled image (image 2); upsampling with a patch-based supervised approach (image 3); upsampling with sparse-dense sampling (image 4).}
\label{Ex:4}
\end{figure}

\section{Conclusion}
\label{sec:conclusion}

In this work, we developed a locally supervised upsampling method that exploits invariances naturally present in images for global image restoration. 
The key idea behind local supervision is to leverage these invariances across different regions of an image. 
This implies that supervision data are required only on a small subregion of the full image. 
Unlike a purely patch-wise strategy, the sparse-dense approach allows the network to learn global information beyond individual patches. 
We demonstrate the effectiveness of this method for sparse-dense sampling in OR-PAM, clearly showing the superiority of the sparse-dense strategy.

Future work will focus on studying the influence of noise and exploring other invariances, such as rotation and mirroring. {\color{black} The experimental realization currently relies on translation equivariance, and extensions to rotations and other transforms would likely require different architectures and strategies.} To investigate the boundaries of our method, we will also examine how far the sampling rate can be reduced; for example, recording only every third or fourth pixel in the observation and reducing the size of the supervision region.  
Finally, we plan to extend our method to random sampling and self-supervision with noisy data, where not all samples are available.

\section{Acknowledgments}

This work has been supported by the Austrian Research Funding Association (FFG) within the Project \textit{Compressed Photoacoustic Remote Sensing} (project number FO999898886).

\end{document}